%% file: oll-perm.tex
\documentclass[nonacm,sigconf]{acmart}

\input{parts/0-packages}
\input{parts/0-acm-config}

\newcommand{\arxivextra}[1]{#1}

\begin{document}
\title{The \ollh\ Genetic Algorithm for Permutations}

\author{Anton Bassin}
\affiliation{%
  \institution{ITMO University}
  \city{Saint Petersburg} 
  \state{Russia} 
  \postcode{197101}
}
\email{anton.bassin@gmail.com}

\author{Maxim Buzdalov}
\affiliation{%
  \institution{ITMO University}
  \city{Saint Petersburg} 
  \state{Russia} 
  \postcode{197101}
}
\email{mbuzdalov@gmail.com}

\begin{abstract}
The \oll~genetic algorithm is a bright example of an evolutionary algorithm which was developed based on the insights from theoretical findings.
This algorithm uses crossover, and it was shown to asymptotically outperform all mutation-based evolutionary algorithms even on simple problems like \textsc{OneMax}.
Subsequently it was studied on a number of other problems, but all of these were pseudo-Boolean.

We aim at improving this situation by proposing an adaptation of the \oll~genetic algorithm to permutation-based problems.
Such an adaptation is required, because permutations are noticeably different from bit strings in some key aspects, such as the number of possible mutations and their mutual dependence.
We also present the first runtime analysis of this algorithm on a permutation-based problem called \ham\ whose properties resemble those of \textsc{OneMax}.
On this problem, where the simple mutation-based algorithms have the running time of $\Theta(n^2 \log n)$ for problem size $n$,
the \oll~genetic algorithm finds the optimum in $O(n^2)$ fitness queries. We augment this analysis with experiments,
which show that this algorithm is also fast in practice.
\end{abstract}


\maketitle

\input{parts/1-introduction}
\input{parts/2-preliminaries}
\input{parts/3-algorithm}
\input{parts/4-theory}
\input{parts/5-experiments}
\input{parts/z-conclusion}

\bibliographystyle{ACM-Reference-Format}
\bibliography{../../../bibliography} 

\end{document}

%% file: parts/0-packages.tex
\graphicspath{{generated-images/}}

\usepackage{tikz}
\usetikzlibrary{external}
\usepackage{algorithm}
\usepackage{algpseudocode}

\usepackage{booktabs}
\usepackage{pgfplots}
\pgfplotsset{compat=newest}
\pgfdeclarelayer{background}
\pgfsetlayers{background,main}

\pgfplotscreateplotcyclelist{myplotcycle}{%
blue,mark=*\\                 %
red!80!black,mark=*\\         %
violet!80!black,mark=*\\      %
cyan!80!black,dashed,mark=*\\ %
black,dashed,mark=*\\         %
green!70!black,mark=*\\       %
blue,dashed,mark=*\\          %
orange,mark=*\\               %
black,mark=*\\                %
}

\newcommand{\diffcolor}[1]{\textcolor{blue}{#1}}

\newcommand{\ea}{$(1+1)$~EA}
\newcommand{\eah}{\texorpdfstring{$(1+1)$~EA}{(1+1)~EA}}
\newcommand{\oll}{$(1+(\lambda,\lambda))$}
\newcommand{\ollga}{$(1+(\lambda,\lambda))$~GA}
\newcommand{\ollh}{\texorpdfstring{\oll}{(1+(lambda,lambda))}}
\newcommand{\ham}{\textsc{Ham}}

%% file: parts/0-acm-config.tex
\copyrightyear{2020}
\acmYear{2020}
\setcopyright{acmlicensed}
\acmConference[GECCO '20 Companion]{Genetic and Evolutionary Computation Conference Companion}{July 8--12, 2020}{Canc\'un, Mexico}
\acmBooktitle{Genetic and Evolutionary Computation Conference Companion (GECCO '20 Companion), July 8--12, 2020, Canc\'un, Mexico}
\acmPrice{15.00}
\acmDOI{10.1145/3377929.3398148}
\acmISBN{978-1-4503-7127-8/20/07}

%% file: parts/1-introduction.tex
\section{Introduction}

The \oll~genetic algorithm (GA), proposed in~\cite{learning-from-black-box-thcs}, is a fairly recent algorithm with very interesting properties.
It was the first general-purpose optimizer to outperform simple evolutionary algorithms, such as the $(1+1)$ evolutionary algorithm (EA),
on the simple benchmark problem \textsc{OneMax} asymptotically: compared to the required $\Theta(n \log n)$ fitness evaluations in expectation, typical to hill-climbers,
it needs only $o(n \sqrt{\log n})$ of them for an optimal fixed parameter setting,
and $O(n)$ when the simple $1/5$-th rule is used to control the parameter $\lambda$~\cite{doerr-doerr-lambda-lambda-parameters-journal}.

On a larger scale, the performance of the \ollga\ on \textsc{OneMax} is asymptotically better than of any unbiased black-box algorithm which uses only mutations. 
This can be seen as a positive answer to the question of whether the crossover, as a design pattern in evolutionary algorithms, is useful even for simple problems:
in artificial settings, it was known for quite a while~\cite{jansen-crossover}, which was then followed by some evidence on combinatorial problems~\cite{sudholt-ising-models},
while from the complexity perspectives the usefulness of higher-arity operators was shown already in~\cite{doerr-johannsen-faster-blackbox}.
The steady-state $(\mu+1)$~GA has also been shown to outperform mutation-only algorithms~\cite{corusO-standard-steady-state-ga-hillclimbs-faster,corusO-benefits-of-populations},
although not asymptotically, and~\cite{olivetoSW-lower-bound-standard-steady-state} clarified that non-trivial population sizes $\mu$ are crucial for that.
Similar statements for a diversity-preserving $(2+1)$~GA were proven in~\cite{sudholt-crossover-speeds-up-evco,pinto-doerr-crossover-ppsn18}.

The \ollga\ was also experimentally found to be competitive in application to satisfiability problems~\cite{goldman-punch-ppp},
which was subsequently confirmed theoretically~\cite{buzdalovD-gecco17-3cnf} based on some of the insights from the earlier works~\cite{doerr-neumann-sutton-oneplusone-cnf,SuttonN14}. 
This algorithm has also been analysed on another benchmark problem, \textsc{LeadingOnes}~\cite{AntipovDK19}.
However, its theoretical and practical applications were so far limited to the domain of pseudo-Boolean functions.
The main reason for this seems to be the structure of the \ollga, which makes it more efficient in certain circumstances,
but, in contrast to many other evolutionary algorithms, also complicates introducing changes, such as switching to a different domain.

We aim at changing this situation. We propose an adaptation of the \ollga\ to a different kind of the search space, the permutations.
First, this is one of the most popular representations that has very different characteristics compared to bit strings and to real-valued vectors,
and in particular, it allows much more structurally different choices to what a mutation and a crossover can be (cf.~\cite[Chapter 17]{back-evolutionary-computation-1}).
Second, permutations are the search space for one of the most studied NP-hard combinatorial optimization problem with direct practical applications, the traveling salesperson problem.
For this problem, crossover was shown to be very important by the field of gray-box optimization~\cite{sanches-whitley-tinos-tsp-px}, playing roughly the same role here as for
pseudo-Boolean functions~\cite{whitley-chicano-goldman,whitley-one-million-variables-nk-gecco17}.
Finally, there are plenty of practical applications that require to optimize very complicated functions on permutations, see e.g.~\cite{feoktistov-field-trials-cec2017},
for which the gray-box optimization cannot be an answer.

Permutations have been also considered by the theory of evolutionary algorithms, but to lesser extent compared to bit strings.
The seminal paper~\cite{scharnow-tinnefeld-wegener-EA-sorting} proves some important facts for various flavors of the \ea\ on the problem of sorting a permutation.
In this context, a number of different fitness function to evaluate sortedness were considered, as well as a number of distinct mutation operators.
The fitness level method~\cite{Wegener02EvOpt} has also been applied to sorting by inversions in~\cite{level-based-analysis-tevc}.
In~\cite{sorting-by-swaps-noisy}, the performance of hill-climbers was investigated in the presence of comparators with noise.
The fitness landscape of a permutation-based linear ordering problem was investigated in~\cite{ceberio-multiobjectivizing-cops}
with regards to the transformation of a single-objective problem into a multiobjective one.

The main points of our contribution are as follows:
\begin{itemize}
    \item we propose a modification of the \ollga\ suitable to solve problems defined on permutations (Section~\ref{sec:algorithm});
    \item we show that our design is sound by proving for the sorting problem, when using the exchange elementary mutation and the Hamming-distance fitness,
          that the \ollga, with a suitable fitness-dependent choice of $\lambda$, solves it in expected $O(n^2)$ fitness queries,
          faster than $\Theta(n^2 \log n)$ achieved by hill-climbers (Section~\ref{sec:theory});
    \item we augment the proof by the experimental investigation of the running time, which shows that our modification of the \ollga\ is quite efficient in practice as well,
          and the impact of the parameter values on the running time to showcase that the \ollga\ on permutations is a rather interesting subject for theoretical research (Section~\ref{sec:experiments}).
\end{itemize}

\arxivextra{
Our theoretically proven upper bounds use five different modes of the parameter setting for $\lambda$
depending on the fitness value. In our upper bounds, the extreme fitness values
and the narrow adjacent regions require a fixed parameter that depends on the problem size only;
fitness values close to $n/2$ are solved using a mode similar to local search;
finally, the two intermediate regions require a fitness-dependent parameter value.
Their relation with the experimental parameter landscape analysis is twofold:
the boundary effects can be clearly seen, including the visible shelf near to the
opposite of the global optimum, however, the optimal middle-fitness regime is very different from
a local optimizer mode. As a result, the \ollga\ with $\lambda = 2 \log(n+1)$ apparently solves the problem in $O(n^2)$,
which is faster than what our current bounds can prove.
}

We note that, although using fitness-dependent values of $\lambda$ is not what a black-box optimizer should do,
our theoretical bound estimates the possible performance of parameter control methods.

Finally we note that, apart from just performing an adaptation of the \ollga\ to some different shape of the search space, we try to summarize the first guidelines on how to do it for this algorithm in general.
More precisely, we think that anyone who would like to adapt the \ollga\ to their own search space would necessary have to perform the same steps as we have done, so our modifications can be used
in such cases as a guideline.

%% file: parts/2-preliminaries.tex
\section{Preliminaries}

\subsection{Notation}

We use $\log(x)$ for the natural logarithm of $x$.
We denote as $[a..b]$ a set of integers $\{a, a+1, \ldots, b\}$, and $[n]$ is a shortcut for $[1..n]$.
For a real-valued $x$, $\lceil x \rceil$ means $x$ rounded up to the nearest integer.
$\mathcal{B}(n, p)$ stands for the binomial distribution with the number of trials $n$ and success probability $p$.

\subsection{The \eah}

The \ea\ (Algorithm~\ref{algo:ea-orig}) applies, on every iteration,
the standard bit mutation to the parent individual $x$, which is to choose a number $\ell$
from a binomial distribution and to flip $\ell$ randomly chosen bits in $x$. If the resulting offspring $y$ is at least as good as its parent $x$,
it replaces the parent. The default mutation probability is $1/n$, so that the expected value of $\ell$ is exactly one.

In the case when the number of possible mutations is large, and it is computationally infeasible to simulate the binomial distribution by definition,
many researchers approximate it with the Poisson distribution: first, this distribution is sampled for the number of mutations to apply,
and then this number of mutations is sampled without replacement. The expected number of mutations is still constant, which is quite important
both from theoretical and practical viewpoint because local moves are performed sufficiently often.

\begin{algorithm}[!t]
\caption{The \ea\ with standard bit mutation}\label{algo:ea-orig}
\begin{algorithmic}[1]
\State{$x \gets \text{uniformly from }\{0, 1\}^n$}
\For{$t \gets 1, 2, 3, \ldots$}
    \State{$\ell \sim \mathcal{B}(n, 1/n)$}
    \State{$y \gets \text{flip } \ell \text{ uniformly chosen bits in } x$} \Comment{Mutation}
    \If{$f(y) \ge f(x)$} \Comment{Selection}
        \State{$x \gets y$}
    \EndIf
\EndFor
\end{algorithmic}
\end{algorithm}

\subsection{The \ollh~GA}

\begin{algorithm}[!t]
\caption{The \ollga\ for bit strings}\label{algo:ll-orig}
\begin{algorithmic}[1]
\State{$n \gets$ the problem size}
\State{$\lambda \gets$ the population size parameter}
\State{$x \gets \text{uniformly from }\{0, 1\}^n$}
\For{$t \gets 1, 2, 3, \ldots$}
    \State{$p \gets \lambda / n$, $c \gets 1 / \lambda$, $\lambda' \gets \lceil\lambda\rceil$, $\ell \sim \mathcal{B}(n, p)$}
    \For{$i \in [\lambda']$} \Comment{Phase 1: Mutation}
        \State{$x^{(i)} \gets \text{flip } \ell \text{ uniformly chosen bits in } x$}
    \EndFor
    \State{$x' \gets \text{uniformly from } \{x^{(j)} \mid f(x^{(j)}) = \max\{f(x^{(i)})\}\}$}
    \For{$i \in [\lambda']$} \Comment{Phase 2: Crossover}
        \For{$j \in [n]$}
            \State{$y^{(i)}_j \gets x'_j$ with probability $c$, otherwise $x_j$}
        \EndFor
    \EndFor
    \State{$y \gets \text{uniformly from } \{y^{(j)} \mid f(y^{(j)}) = \max\{f(y^{(i)})\}\}$}
    \State{Optionally adjust $\lambda$ based on $f(x)$ and $f(y)$} \Comment{Adaptation}
    \If{$f(y) \ge f(x)$} \Comment{Selection}
        \State{$x \gets y$}
    \EndIf
\EndFor
\end{algorithmic}
\end{algorithm}

Now we describe the \ollga\ for bit strings as proposed in~\cite{learning-from-black-box-thcs}. The algorithm is outlined as Algorithm~\ref{algo:ll-orig}.
In short, this algorithm does the following on each iteration:
\begin{itemize}
    \item during the first phase of each iteration, it creates a intermediate population of size $\lambda' = \lceil\lambda\rceil$ using a higher-than-usual mutation rate of $\lambda / n$;
    \item the offspring $x^{(i)}$ typically have much worse fitness than their parent $x$, however, the best first-phase offspring $x'$ has better chances to contain the new improvements,
          so it is selected for the second phase to undergo crossover with $x$;
    \item in the second phase, the crossover is performed $\lambda'$ times that takes the bits from $x'$ only with probability of $1 / \lambda$,
          so that the outcome of the crossover has one bit different from the parent in expectation;
    \item the best second-phase offspring $y$ competes directly with the parent $x$ as in the \ea, while also optionally adjusting the value of $\lambda$.
\end{itemize}

With an appropriate choice of $\lambda$, the \ollga\ is able to test $\Theta(\lambda^2)$ bits with only $\Theta(\lambda)$ fitness queries,
and the moderate deviations from the optimal choice still retain good performance. For \textsc{OneMax}, the optimal value of $\lambda$,
depending on the fitness $f$, is $\lambda = \sqrt{n / (n - f)}$. If $\lambda$ can depend only on $n$, \cite{doerr-doerr-lambda-lambda-parameters-journal} shows
that the optimal setting for $\lambda$ and the corresponding running time are:
\begin{align*}
\lambda &= \sqrt{\frac{\log n}{\log \log n}},  &T = \Theta\left(n \sqrt{\frac{\log n \log \log \log n}{\log \log n}}\right).
\end{align*}
On the other hand, the simple $1/5$-th rule can be applied to adjust $\lambda$ on-the-fly: if the parent is replaced by an individual with better fitness,
$\lambda$ is divided by a constant $F \in (1;2)$, otherwise it is multiplied by $F^{1/4}$.
The \ollga\ with this rule achieves the provable $O(n)$ time on \textsc{OneMax}, which is strictly better than what is possible with the fixed $\lambda$.

\subsection{Mutations for permutations}

For the sake of self-containedness, we present a list of the most common operations used to introduce small changes to permutations, taken from~\cite{scharnow-tinnefeld-wegener-EA-sorting}.
Just like that paper, we limit ourselves to minimal local changes, which we call \emph{elementary mutations} from now on. When an algorithm needs a global mutation, it can either
use the binomial distribution or Poisson distribution to sample the number of mutations to apply in order.
Below, we list these elementary mutations together with the number of possible mutations for the problem size $n$.
\begin{itemize}
    \item The \emph{exchange mutation}: exchange the elements at two different indices $i$ and $j$. There are $\frac{n(n-1)}{2}$ different mutations.
    \item The \emph{reverse mutation}: reverse a segment of the permutation between indices $i$ and $j$. Regarding the traveling salesperson problem, this mutation is an equivalent of the 2-OPT move.
          Depending on which reversals make sense, there may be up to $n(n-1)$ different mutations, but in any case this number will be $\Theta(n^2)$.
    \item The \emph{jump mutation}: move an element at index $i$ to index $j$ while the elements at the intermediate indices shift to the corresponding direction. There are $n(n-1)$ different mutations.
\end{itemize}

Note that, unlike the common mutations for bit strings, most of the presented mutations can be not commutative,
that is, when having two elementary mutations $m_1$ and $m_2$, it makes a difference whether one applies $m_1$ and then $m_2$, or $m_2$ and then $m_1$.
Among other things, this poses a difficulty for the \ollga, because it is now much less trivial to identify a good mutation if we suspect there is one.

\subsection{Problems \textsc{OneMax} and \ham}

We often refer to the \textsc{OneMax} problem, which is defined on bit strings of length $n$ as a maximization problem as follows:
\begin{equation*}\textsc{OneMax}(x) \mapsto \sum_{i=1}^n [x_i = 1].\end{equation*}

For benchmarking our modification of the \ollga\ we use a permutation-based problem which was called \ham\ in~\cite{scharnow-tinnefeld-wegener-EA-sorting}.
This problem can be formally defined as follows:
\begin{equation*}
\ham(\pi) = \sum_{i=1}^n [\pi_i = i].
\end{equation*}
where $n$ is the problem size. Informally, it counts the number of positions at which the queried permutation coincides with the identity (sorted) permutation.
As our algorithms are unbiased as in~\cite{generic-unbiased-algorithms},
our analysis covers the whole class of similar problems, namely, $\ham_{p}(\pi) = \sum_{i=1}^{n} [\pi_i = p_i]$,
where $p$ is a ``guessed'' permutation.

In this paper, we use the \emph{exchange mutations} as the only considered elementary mutations. These make a good starting point, because from the point of the \ham\ problem
these elementary mutations are the most local operations possible. A number of considerations in this paper would also make sense for different elementary mutations, and of course for other problems,
however, the theoretical analysis and experimental investigations are performed solely for \ham\ and exchange mutations.

We shall now shortly reconsider the basic properties of elementary exchange mutations, which are detailed in~\cite{scharnow-tinnefeld-wegener-EA-sorting}.
The number of elementary exchange mutations is $\binom{n}{2} = \frac{n(n-1)}{2}$.
Assuming that the current Hamming distance to the optimum is $d$,
these elementary mutations can be classified as follows, assuming we evaluate the effect of each mutation independently of others.
\begin{itemize}
    \item Mutations which exchange two positions that were guessed right. There are $\binom{n-d}{2} = \frac{(n-d)(n-d-1)}{2}$ mutations of this sort,
          and each of them increases the Hamming distance by 2, because none of the new positions become guessed right.
    \item Mutations which exchange a position that was guessed right and a position that was guessed wrong. There are $(n-d)d$ mutations of this sort,
          and each of them increases the Hamming distance by 1, because none of the new positions become guessed right.
    \item Mutations which exchange two positions that were guessed wrong. There are $\binom{d}{2} = \frac{d(d-1)}{2}$ mutations of this sort.
          Depending on what happens, the following effects apply:
          \begin{itemize}
              \item both new positions are still guessed wrong: this does not change the Hamming distance;
              \item one of the new positions becomes guessed right: this decreases the Hamming distance by 1;
              \item both new positions becomes guessed right: this decreases the Hamming distance by 2.
          \end{itemize}
\end{itemize}

It is clear that for each position that is guessed wrong there is exactly one mutation that makes it right.
However, if a mutation decreases the Hamming distance by 2, it is counted twice. If there are $x$ such mutations,
there are $d - 2x$ mutations that decrease the Hamming distance by 1. Hence there are at least $\lceil\frac{d}{2}\rceil$
and at most $d$ elementary mutations, $\Theta(d)$ in total, that decrease the Hamming distance.
The expected distance decrease, assuming every elementary mutation is chosen with equal probability,
is exactly $\frac{2d}{n(n-1)}$.

The latter consideration allows to easily re-prove the results from~\cite[Theorem 3]{scharnow-tinnefeld-wegener-EA-sorting}
using suitable modern tools, such as multiplicative drift theorems~\cite{multiplicative-drift-theorem,koetzingK-first-hitting-times-thcs19} for upper and lower bounds.

The effect of applying several elementary exchange mutations may not cumulate if these mutations modify a certain position more than once.
This effect can result in deviations of both signs, as will be shown in the following examples.
For convenience we denote a mutation that exchanges positions $i$ and $j$ by $\langle i, j \rangle$.
In the few examples below, we work with permutations of size 3 and assume that the optimum is $[1, 2, 3]$.
\begin{itemize}
\item Consider a permutation $[2, 1, 3]$. A mutation $\langle 1, 2\rangle$ decreases the Hamming distance by 2, and a mutation $\langle 1, 3 \rangle$ increases it by 1.
      However, applying $\langle 1, 2\rangle$ and then $\langle 1, 3\rangle$ together retains the Hamming distance unchanged (and does not decrease it by 1).
      Applying them in a different order results in a different permutation, however, the Hamming distance is not decreased by 1 as well.
\item Consider a permutation $[2, 3, 1]$. A mutation $\langle 1, 3\rangle$ decreases the Hamming distance by 1, and a mutation $\langle 2, 3 \rangle$ does not change it.
      However, applying $\langle 1, 3\rangle$ and then $\langle 2, 3 \rangle$ decreases the Hamming distance by 3 (not by 1),
      and applying them in a different order increases it by 1.
\end{itemize}

As a result, we see that whenever we have a number of offspring and elementary mutations overlap in some of them,
it is difficult for an algorithm with a structure similar to the \ollga\ to reliably determine which of these offspring contains a promising mutation.

%% file: parts/3-algorithm.tex
\section{The \ollh~GA for Permutations}\label{sec:algorithm}

In this section, we present our modification of the \ollga\ suitable for solving optimization problems on permutations
and discuss the corresponding design choices and their consequences. The pseudocode of the algorithm is outlined in
Algorithm~\ref{algo:ll-perm} with the differences highlighted in \diffcolor{blue}.

The key differences between this modification and the original \ollga\ for bit strings are summarized below.
\begin{enumerate}
    \item Most parameters of the algorithm now depend on the number of possible mutations $m$ (for permutations, $m = \Theta(n^2)$) rather than the problem size $n$.
          In fact, it is just a pure coincidence that the number of mutations and the problem size coincide for the most investigated
          problems on bit strings. Technically, there may exist problems which would benefit from being able to apply more than $n$
          elementary mutations at once; with the default choice, this would not be possible.
    \item The order of elementary mutations matters. This may introduce an additional implementation detail in the algorithm.
          For example, there exist certain ways applicable to bit strings that allow generation of the bit flip indices without
          sampling the (pseudo)random number generator for every bit, which are based on a quite nice distribution of the distance
          between the successive indices~\cite{jansenZ-algorithm-engineering}.
          If the elementary mutations are encoded as integer numbers, and such a method is applied,
          it would generate the indices in an increasing order, which would severely alter the distribution of elementary mutation lists.
          For this reason, an explicit shuffle of elementary mutations may be a good recommendation.
    \item The lists of elementary mutations are stored along the mutants and are, in fact, their complete synonyms. If a problem in hand
          allows incremental fitness re-evaluation, one is no longer required to store the entire mutant in a separate memory:
          a difference is enough, and its size would typically be much smaller than the size of the entire individual.
    \item Crossover is now mutation subsampling: it is performed by taking the ordered list of elementary mutations
          that describes a mutant, picking each elementary mutation with probability $c$ and applying them again to the parent.

          In fact, the original \ollga\ can also be seen this way, and efficient implementations
          which enabled experimenting with problem sizes up to $2^{25}$ in~\cite{buzdalovD-gecco17-3cnf} already use this concept internally.
          For permutations, this also allows efficient implementation if incremental fitness re-evaluation is possible.
          However, it can also be that certain permutation-specific techniques, such as operating on cycles rather than positions,
          can be beneficial, which we leave for the future work.
    \item The order of elementary mutations used in the mutant is preserved in crossover when more than one elementary mutation is chosen.
          This is done in order to reduce the chances of misguiding the fitness-based reproduction selection in the \ollga\ and hence to improve the performance:
          if two non-commuting mutations make a big improvement together, the order-preserving crossover has the bigger probability to take them together.
\end{enumerate}

\begin{algorithm}[!t]
\caption{The \ollga\ for permutations}\label{algo:ll-perm}
\begin{algorithmic}[1]
\State{$n \gets$ the problem size}
\State{$\lambda \gets$ the population size parameter}
\State{\diffcolor{$M \gets$ the set of possible mutations for problem size $n$}}
\State{\diffcolor{$m \gets |M|$}}
\State{$x \gets \text{uniformly from }\diffcolor{\Pi_n}$}
\For{$t \gets 1, 2, 3, \ldots$}
    \State{$p \gets \lambda / \diffcolor{m}$, $c \gets 1 / \lambda$, $\lambda' \gets \lceil\lambda\rceil$, $\ell \sim \mathcal{B}(\diffcolor{m}, p)$}
    \For{$i \in [\lambda']$} \Comment{Phase 1: Mutation}
        \State{\diffcolor{$M^{(i)} \gets \text{sample } \ell \text{ uniformly chosen mutations from } M$}}
        \State{\diffcolor{Shuffle $M^{(i)}$}}
        \State{\diffcolor{$x^{(i)} \gets x \text{ with mutations from } M^{(i)} \text{ applied in order}$}}
    \EndFor
    \State{$x' \gets \text{uniformly from } \{x^{(j)} \mid f(x^{(j)}) = \max\{f(x^{(i)})\}\}$}
    \State{\diffcolor{$M' \gets \text{the corresponding list of mutations}$}}
    \For{$i \in [\lambda']$} \Comment{Phase 2: Crossover}
        \State{\diffcolor{$s \sim \mathcal{B}(\ell, c)$}}
        \State{\diffcolor{$C' \gets s \text{ random elements of } M' \text{ with order preserved}$}}
        \State{\diffcolor{$y^{(i)} \gets x \text{ with mutations from } C' \text{ applied in order}$}}
    \EndFor
    \State{$y \gets \text{uniformly from } \{y^{(j)} \mid f(y^{(j)}) = \max\{f(y^{(i)})\}\}$}
    \State{Optionally adjust $\lambda$ based on $f(x)$ and $f(y)$} \Comment{Adaptation}
    \If{$f(y) \ge f(x)$} \Comment{Selection}
        \State{$x \gets y$}
    \EndIf
\EndFor
\end{algorithmic}
\end{algorithm}

Note that the proposed changes are, in fact, independent from the particular set of mutations and even from the permutation representation itself.
Similar to other evolutionary algorithms, and contrary to the impression that the structure of the \ollga\ is overfitted to bit strings,
the proposed modification may be used for almost arbitrary problem representation.

We are optimistic that our modification can be
used with a large variety of representations and a large variety of possible mutation sets for every such representation.
This includes problems over several permutations, such as various scheduling problems,
or problems defined on binary strings together with mutations that preserve the number of chosen bits, which appear in practice.

%% file: parts/4-theory.tex
\newpage
\section{Running Time Analysis}\label{sec:theory}

In this section, we are going to analyze the \ollga\ using exchange mutations on the \ham\ problem.
We call an elementary mutation \emph{good}
if it improves the fitness when applied to the parent,
and all other elementary mutations we consider to be \emph{bad}.
As long as our generalized modification of the \ollga\ explicitly operates with the lists of elementary mutations,
we do not distinguish a mutant and a list of elementary mutations that generated it.

It is crucial for the \ollga\ to reliably distinguish, by using fitness values only,
the mutants that were constructed using least one good elementary mutation (the \emph{good} mutants)
from those mutants which were constructed using only bad elementary mutations (the \emph{bad} mutants).
The original \ollga\ has no problems with that on \textsc{OneMax}: if the parent's fitness is $f$
and $\ell$ bits are flipped in each mutant, all bad mutants have fitness of $f - \ell$ and all good mutants
have fitness of at least $f - \ell + 2$. For less trivial problems telling good and bad mutants apart
becomes more difficult. This happens, for instance, on MAX-SAT problems~\cite{goldman-punch-ppp,buzdalovD-gecco17-3cnf}
and on linear functions~\cite{bassinB-gecco19-onell-adaptation}, and requires much more involved proofs.

The proof from~\cite{buzdalovD-gecco17-3cnf} uses the following technique. For each iteration of the \ollga\ it introduces
an artificial fitness threshold $\tau$ and pessimistically considers an iteration to be successful only if
the fitness of every good mutant is strictly greater than $\tau$ and the fitness of every bad mutant is strictly less than $\tau$.
This appeared to be easier than using more fine-grained approaches, however, the precision of the resulting bounds is not known
and may be imperfect. Note, however, that the choice of $\tau$ may influence the degree of pessimism,
and whenever the probability of success $p = p(\tau)$ depends on $\tau$, the true success probability is at least the supremum
of $p(\tau)$ across all possible $\tau$.

We use the same idea in our analysis in an even more restricted and simplified form. Recall that every elementary mutation
results in a fitness change belonging to $\{-2,-1,0,+1,+2\}$. We introduce a threshold $\tau \in \{-2,-1,0\}$ and consider
the following definition.

\begin{definition}\label{ordinary-def}
An iteration is \emph{good} with respect to threshold $\tau$ if:
\begin{itemize}
    \item there exists exactly one mutant with at least one good elementary mutation, this mutation uses positions $i_1$ and $i_2$,
          all other elementary mutations use neither of these positions, and they increase the fitness by at least $\tau$ as they are applied;
    \item in all other mutants, all elementary mutations increase the fitness by at most $\tau$ as they are applied.
\end{itemize}
\end{definition}

If an iteration satisfies this definition, the good mutant gets a fitness advantage over all other mutants, thus it is selected for reproduction.
Subsequently, with a constant probability, one of its good elementary mutations is directly applied on the parent.
Note that the ``as they are applied'' clarifications used in Definition~\ref{ordinary-def} make us consider the effect of other mutations in the current context.
In particular, they allow bad elementary mutations to present in the mutants of the first type,
assuming they do not reduce the fitness too much, and even good elementary mutations to present in the mutants of the second type
if their goodness is masked by some other elementary mutations.

Now we consider all $\tau \in \{-2,-1,0\}$ and bound the probabilities that an iteration is good with respect to threshold $\tau$.
We will often use the well-known fact that $(1-1/x)^{x-1} \ge e^{-1}$ for all $x \ge 1$.

\begin{lemma}\label{lemma-t0}
For \ham\ with a problem size $n$, an iteration of the \ollga\ with the parent fitness $f$
and $\lambda$ mutants created by applying $\ell$ elementary mutations is good with respect to threshold $\tau=0$
with probability at least
\begin{equation*}
\frac{\lambda\ell}{n} \cdot e^{-\frac{2\ell(\lambda-1)}{n-3} - \frac{2\ell \min\{n, f + 2\ell + 1\}}{n - \min\{n, f + 2\ell + 1\}}.}
\end{equation*}
\end{lemma}

\begin{proof}
To ease the notation, we always assume that the fitness is always within $[0;n]$ and omit the $\min$ and $\max$ clauses in the corresponding locations.

We first estimate the probability of an individual to belong to the second clause of Definition~\ref{ordinary-def}.
The number of good elementary mutations is at most $n - f$, so the probability of not increasing fitness by applying a single elementary mutation
to an individual with fitness $f$ is
\begin{align*}
    p^{-}_{1} \ge 1 - \frac{2(n - f)}{n(n-1)},
\end{align*}
so, taking into account that each such mutation decreases the fitness by 0, 1 or 2, the probability of not increasing fitness on each of $\ell$ applications
of single elementary mutations is, assuming $n \ge 4$:
\begin{align}
    p^{-}_{\ell} &\ge \prod_{i=0}^{\ell - 1} \left(1 - \frac{2(n - (f - 2i))}{n(n-1)}\right) \ge \left(1 - \frac{2(n - (f - 2\ell))}{n(n-1)}\right)^{\ell} \nonumber\\
                 &\ge \left(1 - \frac{2}{n-1}\right)^{\ell} \ge e^{-\frac{2\ell}{n-3}.} \label{probability-0-negative}
\end{align}

We proceed with estimating the probability of an individual to belong to the first clause of Definition~\ref{ordinary-def}.
The number of good elementary mutations is at least $\frac{n-f}{2}$, and the number of elementary mutations which do not decrease fitness
and do not touch the positions affected by a given good elementary mutation is exactly $(n-(f+2))(n-(f+3))/2$.
Taking into account that each such mutation can only increase the fitness by 0, 1 or 2, we get the following bound for the probability of generating such an individual:
\begin{align}
    p^{+}_{\ell} &\ge \sum_{i=0}^{\ell - 1} \left(\frac{n-(f+2i)}{n(n-1)} \cdot \prod_{j=0}^{\ell-2} \frac{(n-(f+2j+2))(n-(f+2j+3))}{n(n-1)}\right) \nonumber\\
                 &\ge \sum_{i=0}^{\ell - 1} \frac{(n-(f+2\ell-1))^{2\ell-1} }{n^{\ell} (n-1)^{\ell}}
                  = \frac{\ell \cdot (n-(f+2\ell+1))^{2\ell-1}}{n^{\ell} (n-1)^{\ell}} \nonumber\\
                 &\ge \frac{\ell}{n} \left(1 - \frac{f+2\ell+1}{n}\right)^{2\ell} \ge \frac{\ell}{n} \cdot e^{-\frac{2\ell (f + 2\ell + 1)}{n - f - 2\ell - 1}.} \label{probability-0-positive}
\end{align}

Now we get together \eqref{probability-0-negative} and \eqref{probability-0-positive}, remembering that there is exactly one first-clause individual
(which can appear as any of the $\lambda$ individuals) and $\lambda-1$ second-clause individuals.
\begin{align*}
    p_{\tau=0} &= \lambda \cdot (p^{-}_{\ell})^{\lambda - 1} \cdot p^{+}_{\ell} \ge \frac{\lambda\ell}{n} \cdot e^{-\frac{2\ell(\lambda-1)}{n-3} - \frac{2\ell (f + 2\ell + 1)}{n - f - 2\ell - 1}.} \qedhere
\end{align*}
\end{proof}

Lemma~\ref{lemma-t0} essentially tells that, whenever $f = O(\sqrt{n})$,
it is sufficient to maintain $\lambda, \ell = \Theta(\sqrt{n})$ in order to have constant progress.
This is indeed quite natural in the very beginning of the optimization process:
when the fitness is small, most elementary mutations are applied to the positions which are not guessed right.

However, if $f = \Theta(n)$, one needs at most polylogarithmic values of $\lambda$ and $\ell$
to achieve inverse polynomial probabilities of being good with respect to $\tau=0$, hence $\tau<0$ is necessary to consider.

\begin{lemma}\label{lemma-t1}
For \ham\ with a problem size $n$, an iteration of the \ollga\ with the parent fitness $f \ge 3$
and $\lambda$ mutants created by applying $\ell$ elementary mutations is good with respect to threshold $\tau=-1$
with probability at least:
\begin{align*}
\frac{\lambda\ell}{n+f-3} &\cdot \left(\frac{f+1}{n-1}\left(2 - \frac{f}{n}\right) - \frac{4\ell}{n-1}\right)^{(\lambda-1)\ell} \\
                          &\cdot \left(1 - \frac{2\ell (n+f-3) + (f-2)(f-3)}{n(n-1)}\right)^{\ell}_{.}
\end{align*}
\end{lemma}

\begin{proof}
Similarly to the previous lemma, we estimate the probability of decreasing the fitness by at least 1 by applying $\ell$ randomly chosen elementary mutations as follows:
\begin{align}
    p^{-}_{\ell} &\ge \prod_{i=0}^{\ell-1} \left(\frac{2(f-2i)}{n-1} - \frac{(f-2i)(f-2i+1)}{n(n-1)}\right) \nonumber\\
                 &\ge \left(\frac{2(f-2\ell+1)}{n-1} - \frac{f(f+1)}{n(n-1)}\right)^{\ell} \nonumber\\
                 &= \left(\frac{f+1}{n-1}\left(2 - \frac{f}{n}\right) - \frac{4\ell}{n-1}\right)^{\ell}_{.} \label{probability-1-negative}
\end{align}
Now we estimate the probability of an individual to belong to the first clause of Definition~\ref{ordinary-def} as follows:
\begin{align}
    p^{+}_{\ell} &\ge \sum_{i=0}^{\ell - 1} \left(\frac{n-(f+2i)}{n(n-1)} \cdot \prod_{j=0}^{\ell-2} \frac{(n-(f+2j+2))(n+(f+2j-3))}{n(n-1)}\right) \nonumber\\
                 &\ge \sum_{i=0}^{\ell - 1} \frac{n-(f+2\ell-2)}{n(n-1)} \cdot \left(\frac{(n-(f+2\ell-2))(n+f-3))}{n(n-1)}\right)^{\ell-1} \nonumber\\
                 &= \frac{\ell}{n+f-3} \left(\frac{(n-(f+2\ell-2))(n+f-3))}{n(n-1)}\right)^{\ell} \nonumber\\
                 &= \frac{\ell}{n+f-3} \left(1 - \frac{2\ell (n+f-3) + (f-2)(f-3)}{n(n-1)}\right)^{\ell}_{.} \label{probability-1-positive}
\end{align}

We prove the lemma by combining \eqref{probability-1-negative} and \eqref{probability-1-positive}, remembering that there is exactly one first-clause individual
(which can appear as any of the $\lambda$ individuals) and $\lambda-1$ second-clause individuals.
\end{proof}

Lemma~\ref{lemma-t1} is noticeably harder to use than Lemma~\ref{lemma-t0}.
However, one can notice that, whenever $c_1 \cdot n \le f \le c_2 \cdot n$ for constants $0 < c_1 < c_2 < 1$,
all the exponentiation bases in both \eqref{probability-1-negative} and \eqref{probability-1-positive}
are bounded by constants that are less than one.
In particular, for $f = n/2$ the probability is roughly $\frac{\lambda\ell}{1.5 n} \cdot 0.75^{\lambda\ell}$,
which is maximized at $\lambda\ell = \log(4/3)$ to $0.176/n$. As a result, we can only bound the progress \ollga\ 
in the middle of optimization by $\Theta(n)$ fitness evaluations in expectation per each single fitness improvement,
thus yielding $O(n^2)$ running time in this region.

\begin{lemma}\label{lemma-t2}
For \ham\ with a problem size $n$, an iteration of the \ollga\ with the parent fitness $f \ge 3$
and $\lambda$ mutants created by applying $\ell$ elementary mutations is good with respect to threshold $\tau=-2$
with probability at least:
\begin{align*}
\frac{\lambda\ell \cdot \max\{1, (n-f)-2(\ell-1)\}}{n(n-1)} e^{-\frac{2(\ell-1)(2n-3)}{(n-2)(n-3)} - \frac{2(\lambda-1)\ell \cdot \min\{n, (n-f) + 2(\ell-1)\}}{n-1 - 2\min\{n,(n-f) + 2(\ell-1)\}}.}
\end{align*}
\end{lemma}

\begin{proof}
Similarly to Lemma~\ref{lemma-t0}, to ease the notation, we always assume that the fitness is always within $[0;n]$ and omit the $\min$ and $\max$ clauses in the corresponding locations.

There are exactly $f(f-1)/2$ elementary mutations which decrease the fitness by 2, hence the probability component
that corresponds to have a bad mutant in the right shape reads as follows:
\begin{align}
    p^{-}_{\ell} &\ge \prod_{i=0}^{\ell - 1} \frac{(f - 2i)(f - 2i - 1)}{n(n-1)}
                  \ge \left(\frac{(f - 2\ell + 2)(f - 2\ell + 1)}{n(n-1)}\right)^{\ell} \nonumber\\
                 &\ge \left(1 - \frac{2(n - f + 2\ell - 2)}{n-1}\right)^{\ell}
                  \ge e^{-\frac{2\ell \cdot (n - f + 2(\ell - 1))}{n-1 - 2((n-f) + 2(\ell-1))}.} \label{probability-2-negative}
\end{align}

On the other hand, good mutants are now allowed to mutate the remaining positions in any possible way, which results in:
\begin{align}
    p^{+}_{\ell} &\ge \sum_{i=0}^{\ell-1} \frac{n-(f+2i)}{n(n-1)} \left(\frac{(n-2)(n-3)}{n(n-1)}\right)^{\ell-1} \nonumber\\
                 &\ge \frac{\ell \cdot ((n-f)-2(\ell-1))}{n(n-1)} \left(1 - \frac{2(2n-3)}{n(n-1)}\right)^{\ell-1} \nonumber\\
                 &\ge \frac{\ell \cdot ((n-f)-2(\ell-1))}{n(n-1)} e^{-\frac{2(\ell-1)(2n-3)}{(n-2)(n-3)}.} \label{probability-2-positive}
\end{align}

We prove the lemma by combining \eqref{probability-2-negative} and \eqref{probability-2-positive}, remembering that there is exactly one first-clause individual
(which can appear as any of the $\lambda$ individuals) and $\lambda-1$ second-clause individuals.
\end{proof}

Lemma~\ref{lemma-t2} is much more similar to Lemma~\ref{lemma-t0} in the shape of the final result, but not entirely symmetrical.
For instance, when the distance to the optimum $n-f$ is $O(1)$, the lower bound on success probability is roughly proportional to
$\frac{\lambda^2}{n^2} e^{-\lambda^3/n}$ assuming $\lambda \approx \ell$.
This is maximized at $\lambda = (2n/3)^{1/3}$, which yields the success probability $\Omega(n^{-4/3})$, so
finding two last positions requires $O(n^{4/3})$ fitness evaluations in expectation with the optimal choice of $\lambda$.

Now we are going to formulate the main result of this section.

\begin{theorem}\label{runtime-best-parameters}
With the optimal fitness-dependent parameters, the running time of the \ollga\ on \ham\ is $O(n^2)$.
\end{theorem}
\begin{proof}
We use the method of fitness levels, hence we pessimistically assume that each fitness improvement is minimum possible.
We consider the following ranges of the fitness value $f$:
\begin{enumerate}
    \item $f = O(\sqrt{n})$. In this case, we set $\lambda = \Theta(\sqrt{n})$. As $\ell = \Theta(\lambda)$ with the constant probability,
          we choose $\tau=0$ and apply Lemma~\ref{lemma-t0}, which yields a constant probability for an iteration to be good.
          In this mode, the algorithm increases the fitness by at least 1 in $O(1)$ iterations and $O(\sqrt{n})$ fitness evaluations.
    \item $f = \Omega(\sqrt{n})$, $f < c_1 n$ for some constant $0 < c_1 < 1/2$. In this case, we still use $\tau=0$ and hence Lemma~\ref{lemma-t0},
          but we set $\lambda = \Theta(n/f)$. The probability that an iteration is good is $\Omega(n/f^2)$, and the expected number of fitness
          evaluations until the fitness update is $O(f^2 / n \cdot n / f) = O(f)$. The total number of fitness evaluations spent in this phase is \[\sum_{f=\sqrt{n}}^{c_1 n} O(f) = O(n^2).\]
    \item $c_1 n \le f \le c_2 n$ for some constants $0 < c_1 < 1/2 < c_2 < 1$. In this case, we choose $\tau=-1$ and use Lemma~\ref{lemma-t1}.
          We set $\lambda = 1$, so that the algorithm essentially performs local search with the improvement probability of $\Omega(1/n)$,
          the expected time until improvement $O(n)$ and the total number of fitness evaluations in this phase being $O(n^2)$.
    \item $c_2 n < f < n - \Theta(n^{1/3})$. We choose $\tau=-2$ and use Lemma~\ref{lemma-t2}.
          From this lemma, the optimal setting $\lambda = \Theta(\sqrt{n/(n - f)})$ can be derived for the current fitness.
          The probability that an iteration is good is $\Theta(\lambda\ell(n-f)/n^2) = \Theta(1/n)$,
          so the expected number of fitness evaluations until an improvement is $\Theta(\sqrt{n^3/(n - f)})$.
          The total number of fitness evaluations spent in this phase is
          \begin{align*}
            \sum_{f=c_2 n}^{n - n^{1/3}} O\left(\frac{n^{3/2}}{\sqrt{n-f}}\right) = O(n^2).
          \end{align*}
    \item $n - \Theta(n^{1/3}) < f$. We still use $\tau=-2$ and Lemma~\ref{lemma-t2}, which this time recommends setting $\lambda = \Theta(n^{1/3})$.
          This leads to the improvement probability of $\Omega(n^{-4/3})$, the expected number of evaluations until improvement $O(n^{5/3})$
          and the total number of evaluations spend in this phase to be $O(n^2)$.
\end{enumerate}

As a result, each range is traversed by the algorithm in time $O(n^2)$. Since the number of ranges is constant, this proves the theorem.
\end{proof}

%% file: parts/5-experiments.tex
\section{Experiments}\label{sec:experiments}

In this section, we present our experimental results. The implementation of all the experiments is available on GitHub\footnote{\url{https://github.com/mbuzdalov/generic-onell}}.
For the sake of performance, we use a set of performance-improving modifications similar to what has been done in~\cite{goldman-punch-ppp}:
\begin{itemize}
    \item the number of elementary mutations for mutants is sampled from a conditional distribution $[\ell \sim \mathcal{B}(m,p) \mid \ell > 0]$;
    \item similarly, the number of elementary mutations to take during crossover is sampled from $[s \sim \mathcal{B}(\ell,c) \mid s > 0]$;
    \item if the crossover offspring takes all mutations from the mutant, it is not re-evaluated and hence not counted towards the number of fitness evaluations.
\end{itemize}
For a survey of some modifications of this sort please refer to~\cite{practice-aware}.

\subsection{Running Times}

\begin{figure*}[!t]
\begin{tikzpicture}
\begin{axis}[width=\textwidth, height=0.3\textheight, xmode=log, log basis x=2, xlabel={Problem size}, ylabel={Evaluations / $n^2$},
             enlarge x limits={abs=5pt}, enlarge y limits=false, ymin=0,
             extra y ticks={2}, grid=major,
             legend style={at={(0.35,0.975)}, anchor=north west},
             legend columns=2,
             cycle list name=myplotcycle]
\input{data/perm-4-17}
\end{axis}
\end{tikzpicture}
\caption{Average running times for various algorithms on \ham}\label{exp:runtimes}
\Description{This picture contains a plot that displays
             the experimentally determined average running times
             for the considered algorithms on the Ham problem.}
\end{figure*}
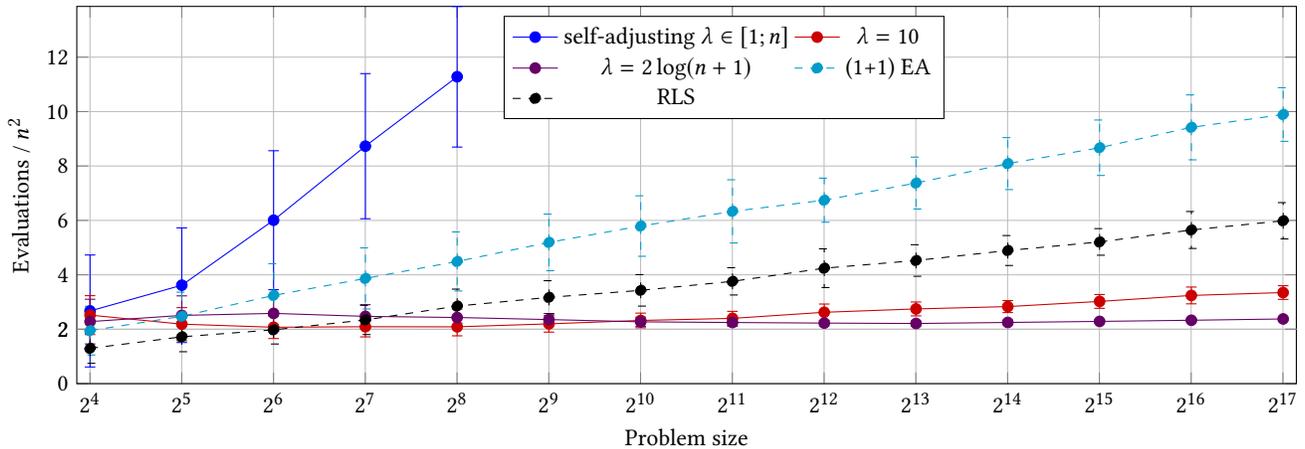

Figure~\ref{exp:runtimes} presents the running times of several algorithms on the \ham\ problem.
To ensure the comparison with popular local search algorithms with the $\Theta(n^2 \log n)$ performance,
we considered the permutation variants of the $(1+1)$ evolutionary algorithm and randomized local search (RLS);
the tested \ea\ also uses the conditionally positive binomial distribution. Since it is unlikely that anyone would ever use
the \ollga\ with the fitness-dependent parameter setting, we considered the following choices:
\begin{itemize}
    \item the static setting $\lambda = 10$;
    \item the default self-adjustment method as in~\cite{doerr-doerr-lambda-lambda-parameters-journal};
    \item the logarithmically capped self-adjustment method as was proposed in~\cite{buzdalovD-gecco17-3cnf};
    \item the problem size dependent static choice $\lambda = 2\log(n+1)$.
\end{itemize}

The logarithmically capped version behaved exactly as the problem size dependent static choice, so we do not plot the former.

The problem sizes were chosen as powers of two from $2^4$ to $2^{17}$; the upper limit was chosen so that the experiments could have been done
in time, as the single run at $n=2^{17}$ reached and exceeded 100 billion fitness evaluations. For every algorithm and every problem size,
100 independent runs were performed; the only exception was the the default self-adjustment method, which started to deteriorate quite early
and so was excluded from the further experimentation.

Figure~\ref{exp:runtimes} presents the running times divided by $n^2$ with the logarithmic abscissa axis and the linear ordinate axis.
In such plots, $\Theta(n^2)$ algorithms produce horizontal plots and $\Theta(n^2 \log n)$ algorithms produce plots with a steady slope.

In Figure~\ref{exp:runtimes} we can easily see that the \ollga\ with the logarithmic choice of $\lambda$ in fact demonstrates a $\Theta(n^2)$ performance.
This is moderately surprising, since the theoretical analysis from Lemma~\ref{lemma-t1} suggests a $\omega(n^2)$ upper bound for the middle range of fitness values
when $\lambda = \Theta(\log n)$. This indicates that a more precise analysis is necessary to fully understand what is happening for these fitness values.
We can also see that this version of the \ollga\ starts outperforming RLS already at a rather small problem size $2^8$, which indicates a high practical efficiency
of the proposed algorithm.

The static parameter version, $\lambda = 10$, can be seen to deviate from the horizontal line towards the higher problem sizes. This is expected, because
theoretical investigations suggest that the performance of this algorithm with a constant population size would be $O((n^2 \log n) / \lambda)$, similar to the \ollga\
on bit strings. Finally, the default self-adjustment method performs much worse than others, which happens because it makes $\lambda$ grow too high in the fitness ranges
where no constant progress is ever possible.

\subsection{Parameter Landscape Analysis}

\newcommand{\xdo}{Distance\phantom{$($}to\phantom{$)$}optimum}
\newcommand{\threedimensionalplot}[4]{%
\begin{tikzpicture}
\begin{axis}[width=#3\textwidth, height=0.2332\textheight, view={0}{90}, shader=interp,
             colormap name=viridis, point meta min=0.1,
             xlabel={#4}, ymode=log, log basis x=2, log basis y=2,
             /pgf/number format/1000 sep={}, #2]
    \addplot3[surf] file {data/perm-3d-#1.txt};
\end{axis}
\end{tikzpicture}}

\begin{figure*}[!t]
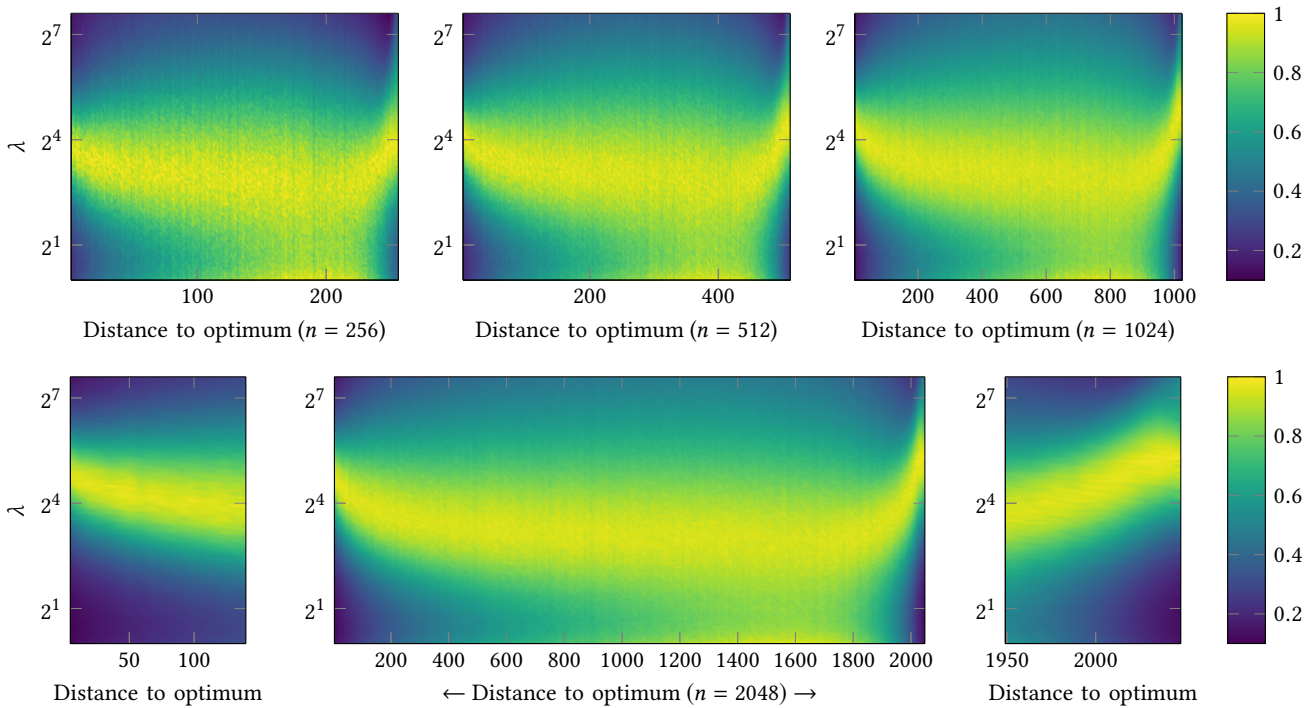

\begin{tabular}{ccc}
\threedimensionalplot{256}{ylabel={$\lambda$}}{0.3333}{\xdo\ ($n=256$)} &
\threedimensionalplot{512}{}{0.3333}{\xdo\ ($n=512$)} &
\threedimensionalplot{1024}{colorbar}{0.3333}{\xdo\ ($n=1024$)}\\
\end{tabular}
\begin{tabular}{ccc}
\threedimensionalplot{2048}{ylabel={$\lambda$}, xmax=140}{0.22}{\xdo} &
\threedimensionalplot{2048}{}{0.53}{$\leftarrow$ \xdo\ $(n=2048)$ $\rightarrow$} &
\threedimensionalplot{2048}{colorbar, xmin=1949}{0.22}{\xdo}
\end{tabular}
\caption{Parameter landscape analysis of the \oll\ GA for \ham}\label{exp:landscape}
\Description{This picture contains plots of relative performance
             of the (1+(lambda,lambda)) genetic algorithm on the Ham problem,
             depending on the distance to the optimum and the value of parameter lambda.}
\end{figure*}

To get a better understanding for how the optimal parameters actually look like, we performed a simple landscape analysis.
For a few problem sizes $n \in \{2^8, 2^9, 2^{10}, 2^{11}\}$ we performed 10000 runs of the \ollga\ with the fixed value of $\lambda$
taken from a multiplicative lattice with a step of $1.05$. For every $\lambda$ and every fitness (or, put alternatively, every possible distance to the optimum)
we approximated the expected number of fitness evaluations until improvement. Figure~\ref{exp:landscape} presents this as a heatmap,
where the color at the point $(d,\lambda)$ stands for the relative probability of improvement compared to the maximum of $(d,\lambda')$ over all tested $\lambda'$.
The yellow color signifies near-optimal values of $\lambda$, whereas colors towards violet indicate inferior values.

Most theoretical insights appear to be visible in the pictures.
In particular, the best values of $\lambda$ for maximal distances appear to be very close to $\sqrt{n}$,
and the distance range of roughly $[2030;2048]$ where the optimal $\lambda$ remains unchanged can be seen in the lower right picture.
The small-distance end also features the best values of $\lambda$ close to the optimum that increase with $n$ roughly as $2n^{1/3}$.

The only unexpected thing is that the optimal $\lambda$ near $f=n/2$ appears to be non-constant, probably logarithmic,
which could not be prediced from our analysis, but which possibly explains the unexpectedly good performance of the $\lambda = 2 \log(n+1)$ version.

%% file: data/perm-4-17.tex
\addplot plot [error bars/.cd, y dir=both, y explicit] coordinates {(16,2.671)+-(0,2.058)(32,3.618)+-(0,2.107)(64,6.007)+-(0,2.554)(128,8.729)+-(0,2.667)(256,11.28)+-(0,2.585)};
\addlegendentry{self-adjusting $\lambda\in[1;n]$};
\addplot plot [error bars/.cd, y dir=both, y explicit] coordinates {(16,2.519)+-(0,0.714)(32,2.187)+-(0,0.6057)(64,2.073)+-(0,0.4152)(128,2.092)+-(0,0.3737)(256,2.088)+-(0,0.3304)(512,2.198)+-(0,0.3037)(1024,2.323)+-(0,0.2693)(2048,2.399)+-(0,0.2582)(4096,2.623)+-(0,0.3028)(8192,2.746)+-(0,0.2556)(16384,2.832)+-(0,0.2207)(32768,3.022)+-(0,0.2535)(65536,3.242)+-(0,0.3085)(131072,3.348)+-(0,0.2512)};
\addlegendentry{$\lambda=10$};
\addplot plot [error bars/.cd, y dir=both, y explicit] coordinates {(16,2.281)+-(0,0.8212)(32,2.508)+-(0,0.7184)(64,2.582)+-(0,0.6585)(128,2.473)+-(0,0.4294)(256,2.431)+-(0,0.3113)(512,2.355)+-(0,0.2227)(1024,2.272)+-(0,0.1714)(2048,2.247)+-(0,0.1387)(4096,2.226)+-(0,0.09956)(8192,2.211)+-(0,0.08888)(16384,2.249)+-(0,0.08323)(32768,2.288)+-(0,0.06416)(65536,2.33)+-(0,0.06085)(131072,2.378)+-(0,0.05441)};
\addlegendentry{$\lambda=2\log (n+1)$};
\addplot plot [error bars/.cd, y dir=both, y explicit] coordinates {(16,1.949)+-(0,0.8955)(32,2.474)+-(0,0.8875)(64,3.243)+-(0,1.166)(128,3.867)+-(0,1.122)(256,4.491)+-(0,1.085)(512,5.192)+-(0,1.04)(1024,5.793)+-(0,1.112)(2048,6.332)+-(0,1.162)(4096,6.743)+-(0,0.8097)(8192,7.373)+-(0,0.9534)(16384,8.088)+-(0,0.9546)(32768,8.671)+-(0,1.019)(65536,9.42)+-(0,1.194)(131072,9.893)+-(0,0.9869)};
\addlegendentry{(1+1) EA};
\addplot plot [error bars/.cd, y dir=both, y explicit] coordinates {(16,1.297)+-(0,0.55)(32,1.721)+-(0,0.5474)(64,1.978)+-(0,0.5207)(128,2.342)+-(0,0.5371)(256,2.85)+-(0,0.6248)(512,3.171)+-(0,0.6135)(1024,3.429)+-(0,0.5786)(2048,3.76)+-(0,0.5011)(4096,4.243)+-(0,0.7125)(8192,4.526)+-(0,0.5799)(16384,4.893)+-(0,0.5502)(32768,5.209)+-(0,0.4888)(65536,5.648)+-(0,0.6795)(131072,5.987)+-(0,0.667)};
\addlegendentry{RLS};

%% file: parts/z-conclusion.tex
\section{Conclusion}

We presented an extension to the \ollga\ that is capable of efficiently solving permutation-based problems,
and whose structure is robust enough to allow applications to completely different problem representations.
We conducted theoretical analysis of this algorithm on the permutation-based problem \ham, the Hamming distance,
which is probably the first theoretical analysis of a black-box crossover-based algorithm on permutation-based problems.

This combination of an algorithm and a problem poses a number of intriguing questions to the theory community.
We have seen five different modes of optimal parameter settings, which require greater parameter values
at both ends of the fitness range, and which is quite far from being perfectly symmetrical. This makes
it a good benchmark problem for self-adjustment methods that should be able to increase and decrease the parameter in a timely manner.
What is more, we still cannot explain why this algorithm is more efficient around the fitness $f = n/2$ than we can show
theoretically. This may force the development of better proof methods inspired by the particular features of the \ollga.
Finally, we hope to investigate the \ollga\ on other permutation-based problems.

\begin{acks}
This research was supported by the Russian Science Foundation, grant number 17-71-20178.
\end{acks}